\documentclass{cambridge6A}
\pdfoutput=1 
\usepackage{natbib}
\usepackage{rotating}
\usepackage{floatpag}
\rotfloatpagestyle{empty}
\usepackage{graphicx}
\usepackage{multind}
\makeindex{subject}

\usepackage{alex}
\usepackage{amsfonts,amsmath,amsthm}
\usepackage{tikz}
\usepackage{mdwlist,url}
\usetikzlibrary{arrows}

\tikzstyle{clear}=[minimum width=0.8cm, minimum height=0.8cm]
\tikzstyle{var}=[draw, minimum width=0.8cm, minimum height=0.8cm]
\tikzstyle{proc}=[fill=yellow, draw, minimum width=0.7cm, minimum height=0.7cm, circle]

\def\R{\mathbb{R}}
\def\E{\mathbb{E}}
\def\Sig{\varSigma}

\newcommand{\myidx}[1]{}
\begin{document}

\title{Parallel Online Learning}

\author{Daniel Hsu, Nikos Karampatziakis, John Langford and Alex J. Smola}

\maketitle

\section{Online Learning}

One well-known general approach to machine learning is to repeatedly
greedily update a partially learned system using a single labeled
data instance.  A canonical example of this is provided by the
perceptron algorithm~\citep{Rosenblatt58} \myidx{Perceptron} 
which modifies a weight
vector by adding or subtracting the features of a misclassified
instance.  More generally, typical methods compute the gradient of the
prediction's loss with respect to the weight vector's parameters, and then
update the system according to the negative gradient.
This basic approach has many variations and extensions, as well as at least
two names.
In the neural network literature, this approach is often called
``stochastic gradient descent'' \myidx{stochastic gradient descent}, while in the learning theory literature it
is typically called ``online gradient descent''\myidx{online gradient descent}.
For the training of complex nonlinear prediction systems, the stochastic
gradient descent approach was described long ago and has been standard
practice for at least two decades~\citep{BHo69,RumHinWil86c,Amari67}.

Algorithm~\ref{alg:GD} describes the basic gradient descent algorithm
we consider here.  The core algorithm uses a differentiable loss
function $\ell = \ell(\cdot,y)$ to measure the quality of a prediction $\hat{y}$ with
respect to a correct prediction $y$, and a sequence of learning rates
$(\eta_t)$.
Qualitatively, a ``learning rate'' \myidx{learning rate} is the degree to which the weight
parameters are adjusted to predict in accordance with a data instance.
For example, a common choice is squared loss where $\ell(\hat{y},y) =
(y-\hat{y})^2$ and a common learning rate sequence is $\eta_t =
1/\sqrt{t}$.

\begin{algorithm}[h]
  \caption{Gradient Descent \label{alg:GD}}
  \begin{algorithmic}
    \STATE {\bf input} loss function $l$, learning rate schedule $\eta_t$
    \STATE {\bf initialize} for all $i \in \{1,\ldots,n\}$, weight$\,\,\,w_i := 0$ 
    \FOR{$t=1$ to $T$}
    \STATE Get next feature vector $x\in \mathbb{R}^n$
    \STATE Compute prediction $\hat{y} := \inner{w}{x}$
    \STATE Get corresponding label $y$
    \STATE For $i \in \{1,\ldots,n\}$ compute gradient $g_i :=
    \frac{\partial \ell(\hat{y},y)}{\partial w_i} \,\, \left( = \frac{\partial \ell(\inner{w}{x},y)}{\partial w_i} \right)$
    \STATE For $i \in \{1,\ldots,n\}$ update $w_i := w_i - \eta_t g_i$
    \ENDFOR
  \end{algorithmic}
\end{algorithm}

There are several basic observations regarding efficiency of online learning approaches.
\begin{itemize*}
  \item At a high level, many learning system make a sequence of
    greedy improvements.  For such systems, it is difficult to reduce
    these improvements to one or just a few steps of greedy
    improvement, simply because the gradient provides local
    information relevant only to closely similar parameterizations,
    while successful prediction is a global property.  This
    observation applies to higher order gradient information such as
    second derivatives as well.  An implication of this observation is
    that multiple steps must be taken, and the most efficient way to
    make multiple steps is to take a step after each instance is
    observed.
\item If the same instance occurs twice in the data, it's useful to
  take advantage of data as it arrives.  Take the extreme case where
  every instance is replicated $n$ times. Here an optimization
  algorithm using fractions of $1/n$ of the data at a time would enjoy
  an $n$-fold speedup relative to an algorithm using full views of the
  data for optimization. While in practice it is difficult to
  ascertain these properties beforehand, it is highly desirable to
  have algorithms which can take advantage of redundancy and
  similarity as data arrives.
\item The process of taking a gradient step is generally amortized by
  prediction itself.
  For instance, with the square loss $\ell(\hat{y},y) = \frac12(\hat{y}-y)^2$, the
  gradient is given by $(\hat{y}-y) x_i$ for $i \in \{1,\ldots,n\}$, so the additional
  cost of the gradient step over the prediction is roughly just a single
  multiply and store per feature.
  Similar amortization can also be achieved with complex nonlinear
  circuit-based functions, for instance, when they are compositions of
  linear predictors.
\item The process of prediction can often be represented in vectorial
  form such that highly optimized linear algebra routines can be
  applied to yield an additional performance improvement.
\end{itemize*}

Both the practice of machine learning and the basic observations above
suggest that gradient descent learning techniques are well suited to
address large scale machine learning problems.  Indeed, the techniques
are so effective, and modern computers are so fast, that we might
imagine no challenge remains. After all, a modern computer might have
8 cores operating at 3GHz, each core capable of $4$ floating point
operations per clock cycle, providing a peak performance of
96GFlops. A large dataset by today's standards is about webscale\myidx{webscale},
perhaps $10^{10}$ instances, each $10^4$ features in size.  Taking the ratio,
this suggests that a well-implemented learning algorithm might be able
to process such a dataset in under 20 minutes. Taking into account
that GPUs\myidx{GPU} are capable of delivering at least one order of magnitude
more computation and that FPGAs\myidx{FPGA} might provide another order of
magnitude, this suggests no serious effort should be required to scale
up learning algorithms, at least for simple linear predictors.

However, considering only floating point performance is
insufficient to capture the constraints imposed by real systems: the
limiting factor is not computation, but rather network limits\myidx{network limits} on
bandwidth and latency.  This chapter is about dealing with these
limits in the context of gradient descent learning algorithms.  We
take as our baseline gradient descent learning algorithm a simple
linear predictor, which we typically train to minimize squared loss.
Nevertheless, we believe our findings with respect to these
limitations qualitatively apply to many other learning algorithms
operating according to gradient descent on large datasets.

Another substantial limit is imposed by label information---it's
difficult in general to cover the cost of labeling $10^9$ instances.
For large datasets relevant to this work, it's typically the case that
label information is derived in some automated fashion---for example a
canonical case is web advertisement where we might have $10^{10}$
advertisements displayed per day, of which some are clicked on and
some are not.

\section{Limits due to Bandwidth and Latency}

The bandwidth limit\myidx{bandwidth constraint} is well-illustrated by the Stochastic Gradient Descent
(SGD) implementation~\citep{Bottou08}. Leon Bottou released it as a
reference implementation along with a classification problem with 781K
instances and 60M total (non-unique) features derived from
RCV1~\citep{LewYanRosLi04}.  On this dataset, the SGD implementation
might take 20 seconds to load the dataset into memory and then learn a
strong predictor in 0.4 seconds.  This illustrates that the process of
loading the data from disk at 15MB/s is clearly the core bottleneck.

But even if that bottleneck were removed we would still be far from
peak performance: 0.4 seconds is about 100 times longer than expected
given the peak computational limits of a modern CPU.  A substantial
part of this slowdown is due to the nature of the data, which is
sparse.  With sparse features, each feature might incur
the latency\myidx{latency constraint} to access either cache or RAM (typically a 10x
penalty)\myidx{cache miss}, imposing many-cycle slowdowns on the computation.  Thus,
performance is sharply limited by bandwidth and latency constraints
which in combination slow down learning substantially.

Luckily, gradient descent style algorithms do not require loading all
data into memory.  Instead one data instance can be loaded, a model updated,
and then the instance discarded.  A basic question is: Can this be done
rapidly enough to be an effective strategy?  For example, a very
reasonable fear is that the process of loading and processing
instances one at a time induces too much latency, slowing the overall approach unacceptably.

The Vowpal Wabbit (VW)\myidx{Vowpal Wabbit} software~\citep{LanLiStr07} provides an existence
proof that it is possible to have a fast fully online implementation
which loads data as it learns.  On the dataset above, VW can load and
learn on the data simultaneously in about 3 seconds, an order of
magnitude faster than SGD.  A number of tricks are required to achieve
this, including a good choice of cache format, asynchronous parsing,
and pipelining of the computation.  A very substantial side benefit of
this style of learning is that we are no longer limited to datasets
which fit into memory.  A dataset can be either streamed from disk or
over the network, implying that the primary bottleneck is bandwidth,
and the learning algorithm can handle datasets with perhaps $10^{12}$
non-unique features in a few hours.

The large discrepancy between bandwidth and available computation
suggests that it should be possible to go beyond simple linear models
without a significant computational penalty: we can compute nonlinear
features of the data and build an extended linear model based on
those features. For instance, we may use the random kitchen sink
features~\citep{RahRec08} to obtain prediction performance comparable
with Gaussian RBF kernel classes. Furthermore, while general
polynomial features are computationally infeasible it is possible to
obtain features based on the outer product\myidx{outer product features} of two sets of features
efficiently by explicitly expanding such features on the fly. These
outer product features can model interaction\myidx{interaction features} between two sources of
information, for example the interaction of (query,result) feature
pairs is often relevant in internet search settings.

VW allows the implicit specification of these outer product features
via specification of the elements of the pairs.  The outer product
features thus need not be read from disk, implying that the disk
bandwidth limit is not imposed.  Instead, a new limit arises based on
random memory access latency and to a lesser extent on bandwidth
constraints. This allows us to perform computation in a space of up to
$10^{13}$ features with a throughput in the order of $10^8$
features/second.  Note that VW can additionally reduce the
dimensionality of each instance using feature
hashing~\citep{ShiPetDroLanetal09,Weinbergeretal09} \myidx{feature hashing}, which is essential
when the (expanded) feature space is large, perhaps even exceeding
memory size.  The core idea here is to use a hash function which
sometimes collides features.  The learning algorithm learns to deal
with these collisions, and the overall learning and evaluation process
happens much more efficiently due to substantial space savings.

This quantity remains up to two orders of magnitude below the
processing limit imposed by a modern CPU (we have up to 100 Flops
available per random memory access). This means that there is plenty
of room to use more sophisticated learning algorithms without
substantially slowing the learning process.  Nevertheless, it also
remains well below the size of the largest datasets, implying that our
pursuit of a very fast efficient algorithm is not yet complete.

To make matters more concrete assume we have datasets of 10TB size
(which is not uncommon for web applications). If we were to stream
this data from disk we cannot expect a data stream of more than
100MB/s per disk (high performance arrays might achieve up to 5x this
throughput, albeit often at a significant CPU utilization). This
implies that we need to wait at least $10^5$ seconds, i.e.\ 30 hours
to process this data on a single computer. This is assuming an optimal
learning algorithm which needs to see each instance only once and a
storage subsystem which is capable of delivering sustained peak
performance for over a day. Even with these unrealistic assumptions
this is often too slow.

\section{Parallelization Strategies}

Creating an online algorithm to process large amounts of data directly
limits the designs possible. In particular, it suggests decomposition
of the data either in terms of instances or in terms of features as
depicted in Figure~\ref{fig:sharding}. Decomposition in terms of
instances automatically reduces the load per computer since we only
need to process and store a fraction of the data on each computer. We
refer to this partitioning as ``instance sharding''\footnote{In the
  context of data, ``shard'' is typically used to define a partition
  without any particular structure other than size.}\myidx{instance sharding}.

An alternative is to decompose data in terms of its features.  While
it does \emph{not} reduce the number of \emph{instances} per computer,
it reduces the data per computer by reducing the number of features
associated with an instance for each computer, thus increasing the
potential throughput per computer.

\begin{figure}[tb]
\begin{center}
\includegraphics[width=0.6\textwidth]{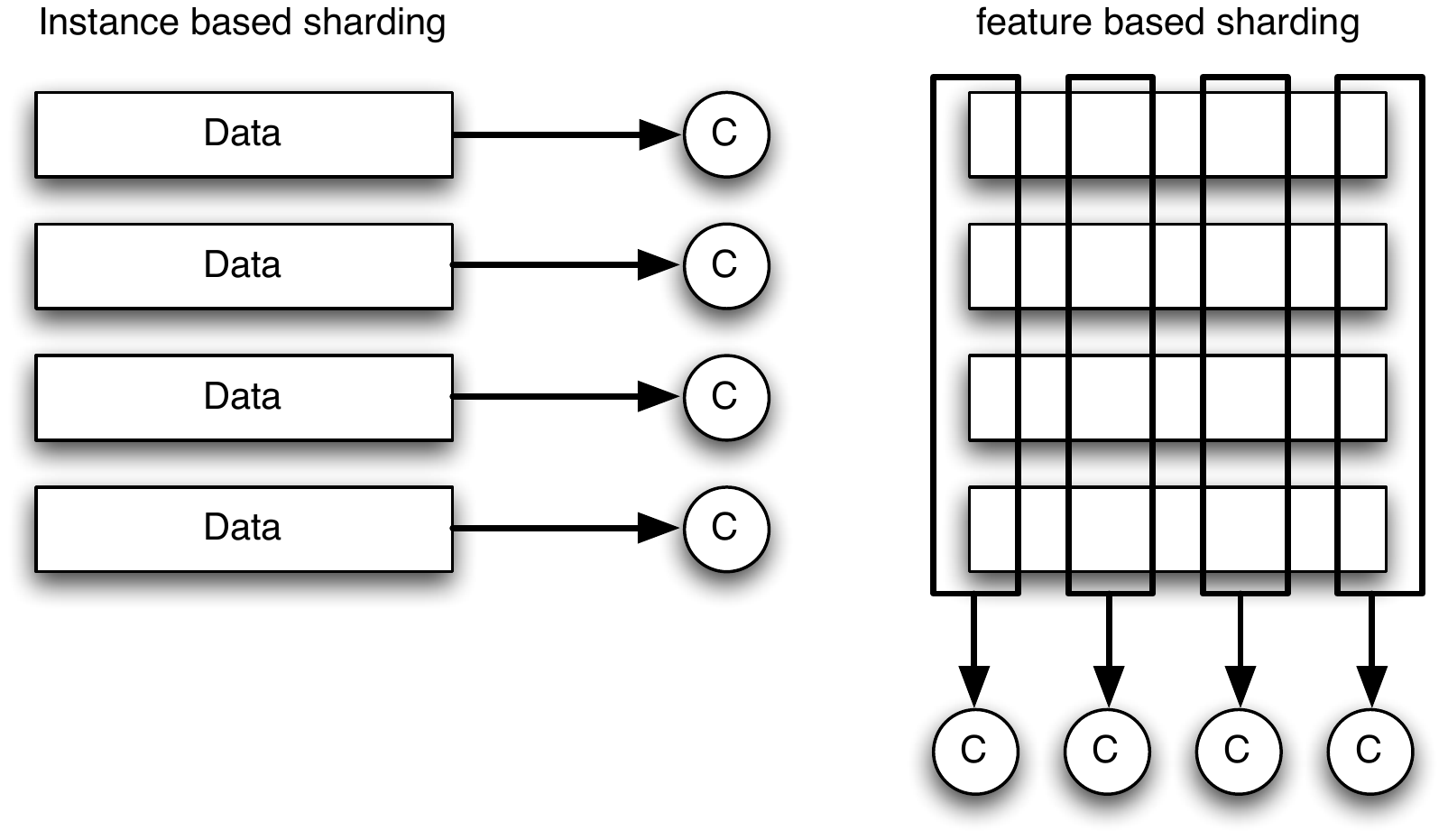}
\end{center}
\caption{Two approaches to data splitting. Left: instance shards,
  Right: feature shards. \label{fig:sharding}}
\end{figure}

A typical instance shard scheme runs the learning algorithm on each
shard, combines the results in some manner, and then runs a learning
algorithm again (perhaps with a different initialization) on each
piece of the data.  An extreme case of the instance shard approach is
given by parallelizing statistical query algorithms~\citep{CKLYBNO07}, which
compute statistics for various queries over the entire dataset, and
then update the learned model based on these queries, but there are
many other variants as well~\citep{MMMSW09,MHM10}.  The instance shard
approach has a great virtue---it's straightforward and easy to
program. 

A basic limitation of the instance shard approach is the
``combination'' operation which does not scale well with model
complexity.  When a predictor is iteratively built based on
statistics, it is easy enough to derive an aggregate statistic.  When
we use an online linear predictor for each instance shard, some
averaging or weighted averaging style of operation is provably
sensible.  However, when a nonlinear predictor is learned, it is
unclear how to combine the results.  Indeed, when a nonlinear
predictor has symmetries, and the symmetries are broken differently on
different instance shards, a simple averaging approach might cancel
the learning away.  An example of a symmetry is provided by a
two-layer neural network with 2 hidden nodes.  By swapping the weights
in the first hidden node with the weights of the second hidden node,
and similarly swapping the weights in the output node we can build a
representationally different predictor with identical predictions.  If
these two neural networks then have their weights averaged, the
resulting predictor can perform very poorly.

We have found a feature shard\myidx{feature sharding} approach more effective after the
(admittedly substantial) complexity of programming has been addressed.
The essential idea in a feature shard approach is that a learning
algorithm runs on a subset of the features of each instance, then the
predictions on each shard are combined to make an overall prediction
for each instance.  In effect, the parameters of the global model are
partitioned over different machines.  One simple reason why the
feature shard approach works well is due to caching effects---any
learned model is distributed across multiple nodes and hence better
fits into the cache of each node.  This combination process can be a
simple addition, or the predictions from each shard can be used as
features for a final prediction process, or the combination could even
be carried out in a hierarchical fashion.  After a prediction is made,
a gradient based update can be made to weights at each node in the
process.  Since we are concerned with datasets less than $10^{12}$ in
size, the bandwidth required to pass a few bytes per instance around is
not prohibitive.

One inevitable side effect of either the instance shard or the feature
shard approach is a delayed update\myidx{delayed updates}, as
explained below.  Let $m$ be the number of instances and $n$ be the
number of computation nodes.  In the instance shard approach the delay
factor is equal to $m/n$, because $m/n$ updates can occur before
information from a previously seen instance is incorporated into the
model.  With the feature shard approach, the latency is generally
smaller, but more dependent on the network architecture.  In the
asymptotic limit when keeping the bandwidth requirements of all nodes
constant, the latency grows as $O(\log(n))$ when the nodes are
arranged in a binary tree hierarchy; in this case, the prediction and
gradient computations are distributed in a divide-and-conquer fashion
and is completed in time proportional to the depth of the recursion,
which is $O(\log(n))$.  In the current implementation of VW, a maximum
latency of $2048$ instances is allowed.  It turns out that any delay
can degrade performance substantially, at least when instances arrive
adversarially, as we outline next.

\section{Delayed Update Analysis}

We have argued that both instance sharding and feature sharding
approaches require delayed updates in a parallel environment.  Here we
state some analysis of the impact of delay, as given by the delayed
gradient descent\myidx{delayed gradient descent} algorithm in Algorithm~\ref{alg:delay}.  We assume that
at time $t$ we observe some instance $x$ with associated label
$y$. Given the instance $x$ we generate some prediction
$\inner{w}{x}$.  Based on this we incur a loss $\ell(\inner{w}{x},y)$ such as
$\frac{1}{2}(y - \inner{w}{x})^2$.

Given this unified representation we consider the following
optimization algorithm template. It differs from
Algorithm~\ref{alg:GD} because the update is delayed by $\tau$ rounds.
This aspect models the delay due to the parallelization strategy for
implementing the gradient descent computation.

\begin{algorithm}[h]
  \caption{Delayed Gradient Descent \label{alg:delay}}
  \begin{algorithmic}
    \STATE {\bfseries Input:} loss function $l$, learning rate $\eta_t$ and delay $\tau \in
    \NN$
    \STATE {\bf initialize} for all $i\in\{1,\ldots,n\}$, weight$\,\,\,w_i := 0$ 
    \STATE Set $x_1 \ldots, x_\tau := 0$ and compute corresponding $g_t$ for $\ell(0,0)$.
    \FOR{$t = \tau + 1$ {\bfseries to} $T + \tau$}
    \STATE Get next feature vector $x\in \mathbb{R}^n$
    \STATE Compute prediction $\hat{y} := \inner{w}{x}$
    \STATE Get corresponding label $y$
    \STATE For $i \in \{1,\ldots,n\}$ compute gradient $g_{t,i} := \frac{\partial \ell(\hat{y},y)}{\partial w_i}$
    \STATE For $i \in \{1,\ldots,n\}$ update $w_i := w_i - \eta_t g_{t-\tau,i}$
    \ENDFOR
  \end{algorithmic}
\end{algorithm}

\subsection{Guarantees}

We focus on the impact of delay on the convergence rate of the weight
vector learned by the algorithm.  Convergence rate is a natural
performance criterion for online learning algorithms, as it
characterizes the trade-off between running time and learning accuracy
(measured specifically in number of instances versus error rate).

Introducing delay between data presentation and updates can lead to a
substantial increase in error rate. Consider the case where we have a
delay of $\tau$ between the time we see an instance and when we are
able to update $w$ based on the instance. If we are shown $\tau$
duplicates of the same data, i.e.\ $x_t, \ldots, x_{t+\tau-1} =
\bar{x}$ in sequence we have no chance of responding to $\bar{x}$ in
time and the algorithm cannot converge to the best weight vector any
faster than $\frac{1}{\tau}$ times the rate of an algorithm which is able to
respond instantly. Note that this holds even if we are told beforehand
that we will see the same instance $\tau$ times.

This simple reasoning shows that for an adversarially chosen sequence
of instances the regret (defined below) induced by a delay of $\tau$
can never be better than that of the equivalent no-delay algorithm whose
convergence speed is reduced by a factor of $\frac{1}{\tau}$. It turns
out that these are the very rates we are able to obtain in the
adversarial setting. On the other hand, in the non-adversarial
setting, we are able to obtain rates which match those of no-delay
algorithms, albeit with a sizable additive constant which depends on
the delay.

The guarantees we provide are formulated in terms of a
\emph{regret}\myidx{regret}, i.e.\ as a discrepancy relative to the
best possible solution $w^*$ defined with knowledge of all
events. Formally, we measure the performance of the algorithm in terms
of
\begin{align}
  \label{eq:regret}
  \mbox{Reg}[\underbrace{w_1, \ldots, w_T}_{=: W}] & := \sum_{t=1}^T \left[
  \ell(\hat{y}_t, y_t) - \ell(\hat{y}^*_t, y_t) \right] \\
  & \text{ where }
  y^*_t = \inner{x_t}{\argmin_{w} \sum_{t'=1}^T \ell(\hat{y}, y_{t'})}
  \nonumber
\end{align}

\begin{theorem}[Worst case guarantees for delayed updates; \citealp{LanSmoZin09}]
  \label{th:delay}
  If $\nbr{w^*} \leq R^2$ and
  the norms of the gradients  $\nabla_w \ell(\inner{w}{x},y)$ are bounded by
  $L$, then
  \begin{equation} \label{eq:delay}
    \mbox{Reg}[W] \leq 4 R L \sqrt{\tau T}
  \end{equation}
  when we choose the learning rate $\eta_t = \frac{R}{L \sqrt{2 \tau t}}$. 
  If, in addition, $\ell(\inner{w}{x},y)$ is strongly convex with modulus of
  convexity $c$ we obtain the guarantee
  \begin{equation*}
    \mbox{Reg}[W] \leq \frac{L^2}{c} \sbr{\tau + 0.5} \log T + C(\tau,L,c)
  \end{equation*}
  with learning rate $\eta_t = \frac{1}{c(t-\tau)}$, where $C$ is a
  function independent of $T$.
\end{theorem}
\noindent
In other words the average error of the algorithm (as normalized by
the number of seen instances) converges at rate $O(\sqrt{\tau/T})$
whenever the loss gradients are bounded and at rate $O(\tau \log T /
T)$ whenever the loss function is strongly convex. This is exactly
what we would expect in the worst case: an adversary may reorder
instances so as to maximally slow down progress. In this case a
parallel algorithm is no faster than a sequential code. While such
extreme cases hardly occur in practice, we have observed
experimentally that for sequentially correlated instances, delays can
rapidly degrade learning.

If subsequent instances are only weakly correlated or IID, it is
possible to prove tighter bounds where the delay does not directly
harm the update~\citep{LanSmoZin09}.  The basic structure of these
bounds is that they have a large delay-dependent initial regret after
which the optimization essentially degenerates into an averaging
process for which delay is immaterial.  These bounds have many
details, but a very crude alternate form of analysis can be done using
sample complexity bounds.  In particular, if we have a set $H$ of
predictors and at each timestep $t$ choose the best predictor on the
first $t-\tau$ timesteps, we can bound the regret to the best
predictor $h$ according to the following:
\begin{theorem}[IID Case for delayed updates]   \label{th:delay2}
  If all losses are in $\{0,1\}$, for all IID data distributions $D$
  over features and labels, for any $\delta$ in $(0,1)$, with probability
  $1-\delta$
  \begin{align}
    \label{eq:delay2}
    \min_{h\in H} \sum_{t=1}^T \left[ \ell(h(x_t),y_t) - \ell(h_t(x_t),y_t) \right] \leq \tau + \sqrt{T \ln \frac{3|H|T}{\delta}} + \sqrt{\frac{T \ln \frac{3}{\delta}}{2}}
  \end{align}
\end{theorem}
\begin{proof}
The proof is a straightforward application of the Hoeffding bound.  At
every timestep $t$, we have $t-\tau$ labeled data instances.  Applying the
Hoeffding bound for every hypothesis $h$, we have that with
probability $2\delta/3|H|T$, $\left| \frac{1}{t-\tau}\sum_{i=1}^{t-\tau}
\ell(h(x_t),y_t) - E_{(x,y)\sim D} \ell(h(x),y) \right| \leq \sqrt{\frac{\ln
    3|H|T/\delta}{2(t-\tau)}}$.  Applying a union bound over all
hypotheses and timesteps implies the same holds with probability at
least $2\delta/3$. The algorithm which chooses the best predictor in
hindsight therefore chooses a predictor with expected loss at most
$\sqrt{\frac{2\ln 3|H|T/\delta}{t-\tau}}$ worse than the best.  Summing
over $T$ timesteps, we get: $\tau + \sqrt{2 \ln 3|H|T/\delta}
\sum_{t=1}^{T-\tau} \frac{1}{\sqrt{t}} \leq \tau + \sqrt{2T \ln
  3|H|T/\delta}$.  This is a bound on an expected regret.  To get a
bound on the actual regret, we can simply apply a Hoeffding bound
again yielding the theorem result.
\end{proof}

\section{Parallel Learning Algorithms}

We have argued that delay is generally bad when doing online learning
(at least in an adversarial setting), and that it is also unavoidable
when parallelizing.  This places us in a bind: How can we create an
effective parallel online learning algorithm?  We'll discuss two
approaches based on multicore and multinode parallelism.

\subsection{Multicore Feature Sharding}

A multicore processor\myidx{multicore processor} consists of multiple CPUs which operate
asynchronously in a shared memory space.  It should be understood that
because multicore parallelization\myidx{multicore parallelization} does not address the primary
bandwidth bottleneck, its usefulness is effectively limited to those
datasets and learning algorithms which require substantial computation
per raw instance used.  In the current implementation, this implies the
use of feature pairing, but there are many learning algorithms more
complex than linear prediction where this trait may also hold.

The first version of Vowpal Wabbit used an instance sharding approach
for multicore learning, where the set of weights and the instance
source was shared between multiple identical threads which each parsed
the instance, made a prediction, and then did an update to the
weights.  This approach was effective for two threads yielding a near
factor-of-2 speedup since parsing of instances required substantial
work.  However, experiments with more threads on more cores yielded no
further speedups due to lock contention.  Before moving on to a
feature sharding approach, we also experimented with a dangerous
parallel programming technique: running with multiple threads that
\emph{do not} lock the weight vector.  This did yield improved speed,
but at a cost in reduced learning rate and nondeterminism which was
unacceptable.

The current implementation of Vowpal Wabbit uses an asynchronous parsing
thread which prepares instances into just the right format for learning
threads, each of which computes a sparse-dense vector product on a
disjoint subset of the features.  The last thread completing this
sparse-dense vector product adds together the results and computes an
update which is then sent to all learning threads to update their
weights, and then the process repeats.  Aside from index definition
related to the core hashing
representation~\citep{ShiPetDroLanetal09,Weinbergeretal09} Vowpal
Wabbit employs, the resulting algorithm is identical to the single
thread implementation.  It should be noted that although processing of
instances is fully synchronous there is a small amount of
nondeterminism between runs due to order-of-addition ambiguities
between threads.  In all our tests, this method of multicore
parallelization yielded virtually identical prediction performance
with negligible overhead compared to non-threaded code and sometimes
substantial speedups.  For example, with 4 learning threads, about a
factor of 3 speedup is observed.

We anticipate that this approach to multicore parallelization will not
scale to large numbers of cores, because the very tight coupling of
parallelization requires low latency between the different cores.
Instead, we believe that multinode parallelization techniques will
ultimately need to be used for multicore parallelization, motivating the next section.

\subsection{Multinode Feature Sharding}

The primary distinction between multicore and multinode
parallelization\myidx{multinode parallelization} is latency, with the latency between nodes many orders
of magnitude larger than for cores.  In particular, the latency
between nodes is commonly much larger than the time to process an
individual instance, implying that any per-instance blocking operation,
as was used for multicore parallelization, is unacceptable.

This latency also implies a many-instance delayed update which, as we
have argued, incurs a risk of substantially degrading performance.  In
an experiment to avoid this risk, we investigated the use of updates
based on information available to only one node in the computation,
where there is \emph{no} delay.  Somewhat surprisingly, this worked
better than our original predictor.

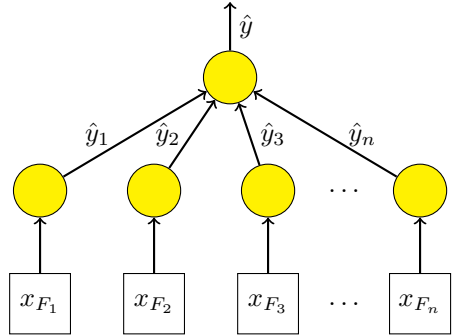
\begin{figure}
\begin{center}
  \begin{tikzpicture}
    \node[var] (a) at (0,0) {$x_{F_1}$};
    \node[var] (b) at (1.5,0) {$x_{F_2}$};
    \node[var] (c) at (3,0) {$x_{F_3}$};
    \node[clear] (d) at (4,0) {$\ldots$};
    \node[var] (e) at (5,0) {$x_{F_n}$};
    \node[proc] (1) at (0,1.5) {};
    \node[proc] (2) at (1.5,1.5) {};
    \node[proc] (3) at (3,1.5) {};
    \node[clear] (4) at (4,1.5) {$\ldots$};
    \node[proc] (5) at (5,1.5) {};
    \node[proc] (6) at (2.5,3) {};
    \draw[->,style=thick] (a) -- (1);
    \draw[->,style=thick] (b) -- (2);
    \draw[->,style=thick] (c) -- (3);
    \draw[->,style=thick] (e) -- (5);
    \draw[->,style=thick] (1) -- (6) node [black,midway,xshift=-14pt] {$\hat{y}_1$};
    \draw[->,style=thick] (2) -- (6) node [black,midway,xshift=-9pt] {$\hat{y}_2$};
    \draw[->,style=thick] (3) -- (6) node [black,midway,xshift=9pt] {$\hat{y}_3$};
    \draw[->,style=thick] (5) -- (6) node [black,midway,xshift=14pt] {$\hat{y}_n$};
    \draw[->,style=thick] (6.north) -- (2.5,4) node [black,midway,xshift=6pt] {$\hat{y}$};
  \end{tikzpicture}
\end{center}
\caption{Architecture for no-delay multinode feature sharding.}
\label{fig:multinode}
\end{figure}

\subsubsection{Tree Architectures}

Our strategy is to employ feature sharding across several nodes, each
of which updates its parameters online as a single-node learning algorithm
would.
So, ignoring the overhead due to creating and distributing the feature
shards (which can be minimized by reorganizing the dataset), we have so
far fully decoupled the computation.
The issue now is that we have $n$ independent predictors each using just a
subset of the features (where $n$ is the number of feature shards), rather
than a single predictor utilizing all of the features.
We reconcile this in the following manner:
(i) we require that each of these nodes compute and transmit a prediction
to a master node\myidx{master node} after receiving each new instance (but before updating its
parameters); and (ii) we use the master node to treat these $n$ predictions
as features, from which the master node learns to predict the label in an
otherwise symmetric manner.
Note that the master node must also receive the label corresponding to each
instance, but this can be handled in various ways with minimal overhead
(e.g., it can be piggybacked with one of the subordinate node's\myidx{subordinate node}
predictions).
The end result, illustrated in Figure~\ref{fig:multinode}, is a two-layer
architecture for online learning with reduced latency at each node and no
delay in parameter updates.

Naturally, the strategy described above can be iterated to create
multi-layered architectures that further reduce the latency at each
node.  At the extreme, the architecture becomes a (full) binary tree:
each leaf node (at the bottom layer) predicts using just a single
feature, and each internal node predicts using the predictions of two
subordinate nodes in the next lower layer as features (see
Figure~\ref{fig:hier-arch}).
Note that each internal node may incur delay proportional to its fan-in
(in-degree), so
reducing fan-in is desirable; however, this comes at the cost of increased
depth and thus prediction latency.
Therefore, in practice the actual architecture that is deployed may be
somewhere in between the binary tree and the two-layer scheme.
Nevertheless, we will study the binary tree structure further because it
illustrates the distinctions relative to a simple linear prediction
architecture.

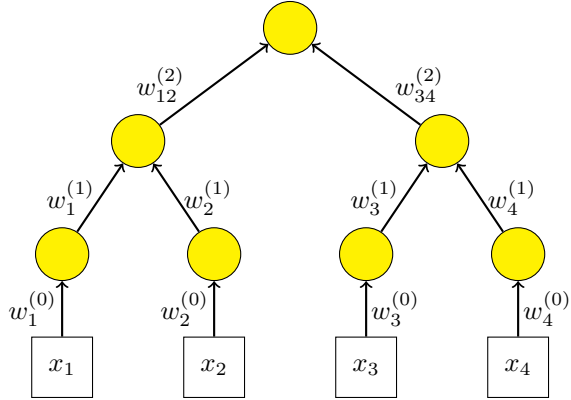
\begin{figure}
\begin{center}
  \begin{tikzpicture}
    \node[var] (a) at (0,0) {$x_1$};
    \node[var] (b) at (2,0) {$x_2$};
    \node[var] (c) at (4,0) {$x_3$};
    \node[var] (d) at (6,0) {$x_4$};
    \node[proc] (1) at (0,1.5) {};
    \node[proc] (2) at (2,1.5) {};
    \node[proc] (3) at (4,1.5) {};
    \node[proc] (4) at (6,1.5) {};
    \node[proc] (5) at (1,3) {};
    \node[proc] (6) at (5,3) {};
    \node[proc] (7) at (3,4.5) {};
    \draw[->,style=thick] (a) -- (1) node [black,midway,xshift=-11pt] {$w_1^{(0)}$};
    \draw[->,style=thick] (b) -- (2) node [black,midway,xshift=-11pt] {$w_2^{(0)}$};
    \draw[->,style=thick] (c) -- (3) node [black,midway,xshift=11pt] {$w_3^{(0)}$};
    \draw[->,style=thick] (d) -- (4) node [black,midway,xshift=11pt] {$w_4^{(0)}$};
    \draw[->,style=thick] (1) -- (5) node [black,midway,xshift=-11pt] {$w_1^{(1)}$};
    \draw[->,style=thick] (2) -- (5) node [black,midway,xshift=12pt] {$w_2^{(1)}$};
    \draw[->,style=thick] (3) -- (6) node [black,midway,xshift=-11pt] {$w_3^{(1)}$};
    \draw[->,style=thick] (4) -- (6) node [black,midway,xshift=12pt] {$w_4^{(1)}$};
    \draw[->,style=thick] (5) -- (7) node [black,midway,xshift=-20pt] {$w_{12}^{(2)}$};
    \draw[->,style=thick] (6) -- (7) node [black,midway,xshift=20pt] {$w_{34}^{(2)}$};
  \end{tikzpicture}
\end{center}
\caption{Hierarchical architecture for no-delay multinode feature sharding.
Each edge is associated with a weight learned by the node at the arrow
head.}
\label{fig:hier-arch}
\end{figure}

\subsubsection{Convergence Time vs Representation Power}

The price of the speed-up that comes with the no-delay approach (even
with the two-layer architecture) is paid in representation power\myidx{representation power}.
That is, the no-delay approach learns restricted forms of linear predictors
relative to what can be learned by ordinary (delayed) gradient descent.
To illustrate this, we compare the resulting predictors from the no-delay
approach with the binary tree architecture and the single-node linear
architecture.
Let $x = (x_1,\ldots,x_n) \in \R^n$ be a random vector
(note that the subscripts now index the features) and $y \in \R$ be a
random variable.  Gradient descent using a linear architecture
converges toward the least-squares linear predictor of $y$ from $x$,
i.e.,
\[ w^*
\ := \ \arg\min_{w \in \R^n} \E\left[\frac12(\inner{x}{w} - y)^2\right]
\ = \ \Sig^{-1} b \in \R^n \]
where
\[ \Sig \ := \ \E[xx^\top] \in \R^{n \times n}
\quad \text{and} \quad
b \ := \ \E[xy] \in \R^n
,
\]
in time roughly linear in the number of features $n$~\citep{KivWar95}.

The gradient descent strategy with the binary tree architecture, on
the other hand, learns weights locally at each node; the weights at
each node therefore converge to weights that are locally optimal for
the input features supplied to the node.
The final predictor is linear in the input features but can differ from the
least-squares solution.
To see this, note first that the leaf nodes learn weights $w_1^{(0)},
\ldots, w_n^{(0)}$, where
\[ w_i^{(0)} := \frac{b_i}{\Sig_{i,i}} \in \R. \]
Then, the $(k+1)$th layer of nodes learns weights from the predictions of
the $k$th layer; recursively, a node whose input features are the
predictions of the $i$th and $j$th nodes from layer $k$ learns the weights
$(w_i^{(k+1)}, w_j^{(k+1)}) \in \R^2$.
By induction, the prediction of the $i$th node in layer $k$ is linear in
the subset $S_i$ of variables that are descendants of this node in the
binary tree.
Let $w_{S_i}^{(k)} \in \R^{|S_i|}$ denote these weights and $x_{S_i} \in
\R^{|S_i|}$ denote the corresponding feature vector.
Then $(w_i^{(k+1)}, w_j^{(k+1)}) \in \R^2$ can be expressed as
\begin{eqnarray*}
\left[\!\!\begin{array}{c}
w_i^{(k+1)} \\
w_j^{(k+1)}
\end{array}\!\!\right]
& = &
\left[\!\!\begin{array}{cc}
\inner{w_{S_i}^{(k)}}{\Sig_{S_i,S_i} {w_{S_i}^{(k)}}}
& \inner{w_{S_i}^{(k)}}{\Sig_{S_i,S_j} {w_{S_j}^{(k)}}}
\\
\inner{w_{S_j}^{(k)}}{\Sig_{S_j,S_i} {w_{S_i}^{(k)}}}
& \inner{w_{S_j}^{(k)}}{\Sig_{S_j,S_j} {w_{S_j}^{(k)}}}
\end{array}\!\!\right]^{-1}
\left[\!\!\begin{array}{c}
\inner{w_{S_i}^{(k)}}{b_{S_i}} \\
\inner{w_{S_j}^{(k)}}{b_{S_j}}
\end{array}\!\!\right]
\end{eqnarray*}
where $\Sig_{S_i,S_j} = \E[x_{S_i} x_{S_j}^\top]$ and $b_{S_i} =
\E[x_{S_i}y]$.
Then, the prediction at this particular node in layer $k+1$ is
\[ w_i^{(k+1)}
\inner{w_{S_i}^{(k)}}{x_{S_i}} + w_j^{(k+1)}
\inner{w_{S_j}^{(k)}}{x_{S_j}}, \]
which is linear in $(x_{S_i},x_{S_j})$.
Therefore, the overall prediction is linear in $x$, with the weight
attached to $x_i$ being a product of weights at the different levels.
However, these weights can differ significantly from $w^*$ when the
features are highly correlated, as the tree architecture only ever
considers correlations between (say) $x_{S_i}$ and $x_{S_j}$ through the
scalar summary $\inner{w_{S_i}^{(k)}}{\Sig_{S_i,S_j} {w_{S_j}^{(k)}}}$.
Thus, the representational expressiveness of the binary tree architecture
is constrained by the local training\myidx{local training} strategy.

The tree predictor can represent solutions with complexities between
Na\"ive Bayes\myidx{Naive Bayes@Na\"ive Bayes} and a linear predictor.  Na\"ive Bayes learns weights
identical to the bottom layer of the binary tree, but stops there and
combines the $n$ individual predictions with a trivial sum: $w_1^{(0)}
x_1 + \ldots + w_n^{(0)} x_n$.  The advantage of Na\"ive Bayes is its
convergence time\myidx{convergence time}: because the weights are learned independently, a
union bound argument implies convergence in just $O(\log n)$ time,
which is exponentially faster than the $O(n)$ convergence time using
the linear architecture!

The convergence time of gradient descent with the binary tree
architecture is roughly $O(\log^2 n)$.  To see this, note that the
$k$th layer converges in roughly $O(\log (n/2^k))$ time since there
are $n/2^k$ parameters that need to converge, plus the time for the
$(k-1)$th layer to converge.  Inductively, this is $O(\log n +
\log(n/2) + \ldots + \log(n/2^k)) = O(k \log n)$.  Thus, all of the
weights have converged by the time the final layer ($k = \log_2 n$)
converges; this gives an overall convergence time of $O(\log^2 n)$.
This is slightly slower than Na\"ive Bayes, but still significantly
faster than the single-node linear architecture.

The advantage of the binary tree architecture over Na\"ive Bayes is that it
can account for variability in the prediction power of various feature
shards, as the following result demonstrates.
\begin{proposition} \label{prop:nb-fails}
There exists a data distribution for which the binary tree architecture can
represent the least-squares linear predictor but Na\"ive Bayes cannot.
\end{proposition}
\begin{proof}
Suppose the data comes from a uniform distribution over the following four
points:
\begin{center}
\begin{tabular}{l|ccc|c}
& $x_1$
& $x_2$
& $x_3$
& $y$
\\
\hline
Point 1
& $+1$ & $+1$ & $-1/2$ & $+1$ \\
Point 2
& $+1$ & $-1$ & $-1$ & $-1$ \\
Point 3
& $-1$ & $-1$ & $-1/2$ & $+1$ \\
Point 4
& $-1$ & $+1$ & $+1$ & $+1$
\end{tabular}
\end{center}
Na\"ive Bayes yields the weights $w = (-1/2,1/2,2/5)$, which incurs mean
squared-error $0.8$.
On the other hand, gradient descent with the binary tree architecture
learns additional weights:
\begin{center}
  \begin{tikzpicture}
    \node[var] (a) at (0,0) {$x_1$};
    \node[var] (b) at (2,0) {$x_2$};
    \node[var] (c) at (4,0) {$x_3$};
    \node[proc] (1) at (0,1.5) {};
    \node[proc] (2) at (2,1.5) {};
    \node[proc] (3) at (4,1.5) {};
    \node[proc] (4) at (1,3) {};
    \node[proc] (5) at (3,3) {};
    \node[proc] (6) at (2,4.5) {};
    \draw[->,style=thick] (a) -- (1) node [black,midway,xshift=-9pt] {$-\frac{1}{2}$};
    \draw[->,style=thick] (b) -- (2) node [black,midway,xshift=5pt] {$\frac{1}{2}$};
    \draw[->,style=thick] (c) -- (3) node [black,midway,xshift=5pt] {$\frac{2}{5}$};
    \draw[->,style=thick] (1) -- (4) node [black,midway,xshift=-9pt] {$1$};
    \draw[->,style=thick] (2) -- (4) node [black,midway,xshift=9pt] {$1$};
    \draw[->,style=thick] (3) -- (5) node [black,midway,xshift=9pt] {$1$};
    \draw[->,style=thick] (4) -- (6) node [black,midway,xshift=-9pt] {$3$};
    \draw[->,style=thick] (5) -- (6) node [black,midway,xshift=9pt] {$-5$};
  \end{tikzpicture}
\end{center}
which ultimately yields an overall weight vector of
\[
(-1/2 \cdot 1 \cdot 3, \ 1/2 \cdot 1 \cdot 3, \ 2/5 \cdot 1 \cdot -5)
\ = \ (-3/2, \ 3/2, \ -2)
\]
which has zero mean squared-error.
\end{proof}
In the proof example, the features are, individually, equally correlated
with the label $y$.
However, the feature $x_3$ is correlated with the two individually
uncorrelated features $x_1$ and $x_2$, but Na\"ive Bayes is unable to
discover this whereas the binary tree architecture can compensate for it.

Of course, as mentioned before, the binary tree architecture (and
Na\"ive Bayes) is weaker than the single-node linear architecture in
expressive power due to its limited accounting of feature correlation.
\begin{proposition} \label{prop:both-fail}
There exists a data distribution for which neither the binary tree
architecture nor Na\"ive Bayes can represent the least-squares linear
predictor.
\end{proposition}
\begin{proof}
Suppose the data comes from a uniform distribution over following four
points:
\begin{center}
\begin{tabular}{l|ccc|c}
& $x_1$
& $x_2$
& $x_3$
& $y$
\\
\hline
Point 1
& $+1$ & $-1$ & $-1$ & $-1$ \\
Point 2
& $-1$ & $+1$ & $-1$ & $-1$ \\
Point 3
& $+1$ & $+1$ & $-1$ & $+1$ \\
Point 4
& $+1$ & $+1$ & $-1$ & $+1$
\end{tabular}
\end{center}
The optimal least-squares linear predictor is the all-ones vector $w^* =
(1,1,1)$ and incurs zero squared-error (since $1\cdot x_1 + 1\cdot x_2
+ 1\cdot x_3=y$ for every point).
However, both Na\"ive Bayes and the binary tree architecture yield weight
vectors in which zero weight is assigned to $x_3$, since $x_3$ is
uncorrelated with $y$; any linear predictor that assigns zero weight to
$x_3$ has expected squared-error at least $1/2$.
\end{proof}

\subsection{Experiments}

Here, we detail experimental results conducted on a medium size
proprietary ad display dataset.  The task associated with the dataset
is to derive a good policy for choosing an ad given user, ad, and page
display features.  This is accomplished via pairwise training
concerning which of two ads was clicked on and element-wise evaluation
with an offline policy evaluator~\citep{ExpSca08}.  There are several
ways to measure the size of this dataset---it is about 100Gbytes when
gzip compressed, has around 10M instances, and about 125G non-unique
nonzero features.  In the experiments, VW was run with $2^{24} \simeq
16M$ weights, which is substantially smaller than the number of unique
features.  This discrepancy is accounted for by the use of a hashing
function, with $2^{24}$ being chosen because it is large enough such
that a larger numbers of weights do not substantially improve results.

In the experimental results, we report the ratio of progressive
validation squared losses~\citep{BluKalLan99} and wall clock times to a
multicore parallelized version of Vowpal Wabbit running on the same
data and the same machines.  Here, the progressive validation squared
loss is the average over $t$ of $(y_t-\hat{y}_t)^2$ where critically,
$\hat{y}_t$ is the prediction just prior to an update.  When data is
independent, this metric has deviations similar to the average loss
computed on held-out evaluation data.

Every node has 8 CPU cores and is connected via gigabit Ethernet.  All
learning results are obtained with single pass learning on the dataset
using learning parameters optimized to control progressive validation
loss.  The precise variant of the multinode architecture we
experimented with is detailed in Figure~\ref{localupdate}.  In particular,
note that we worked with a flat hierarchy using 1-8 feature shards
(internal nodes).
All code is available in the current Vowpal Wabbit open source code
release.

\begin{figure}
a)\includegraphics[scale=0.45]{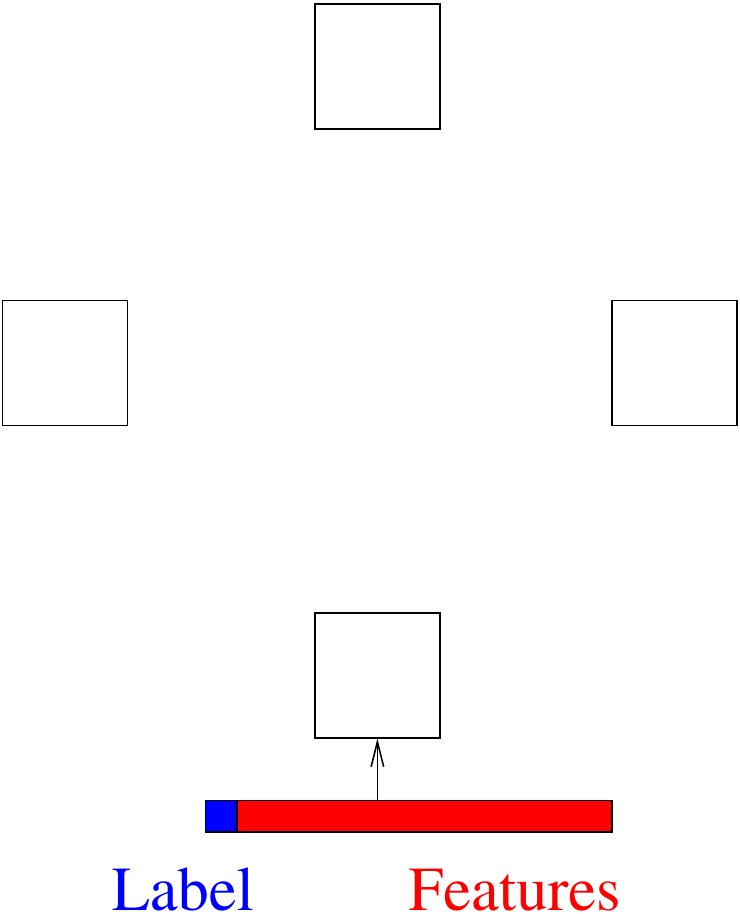}
b)\includegraphics[scale=0.45]{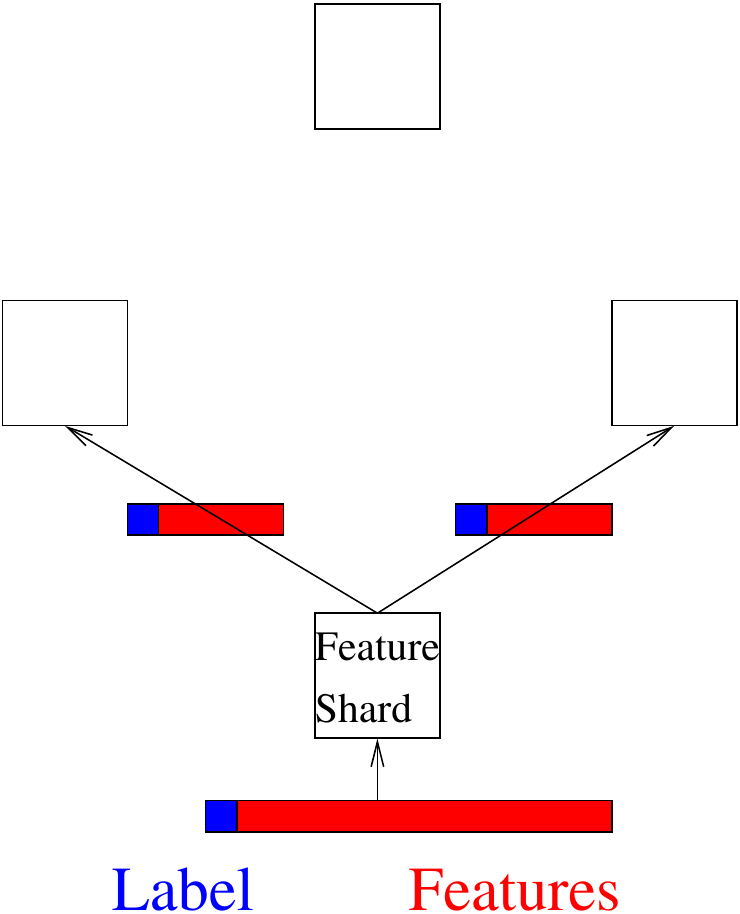}
c)\includegraphics[scale=0.45]{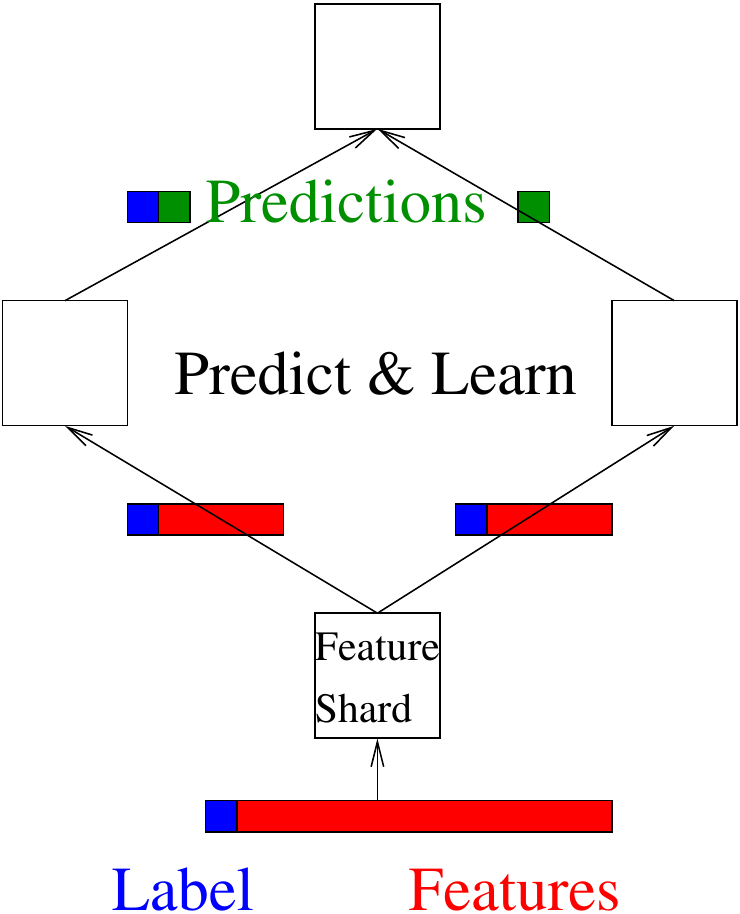}
d)\includegraphics[scale=0.45]{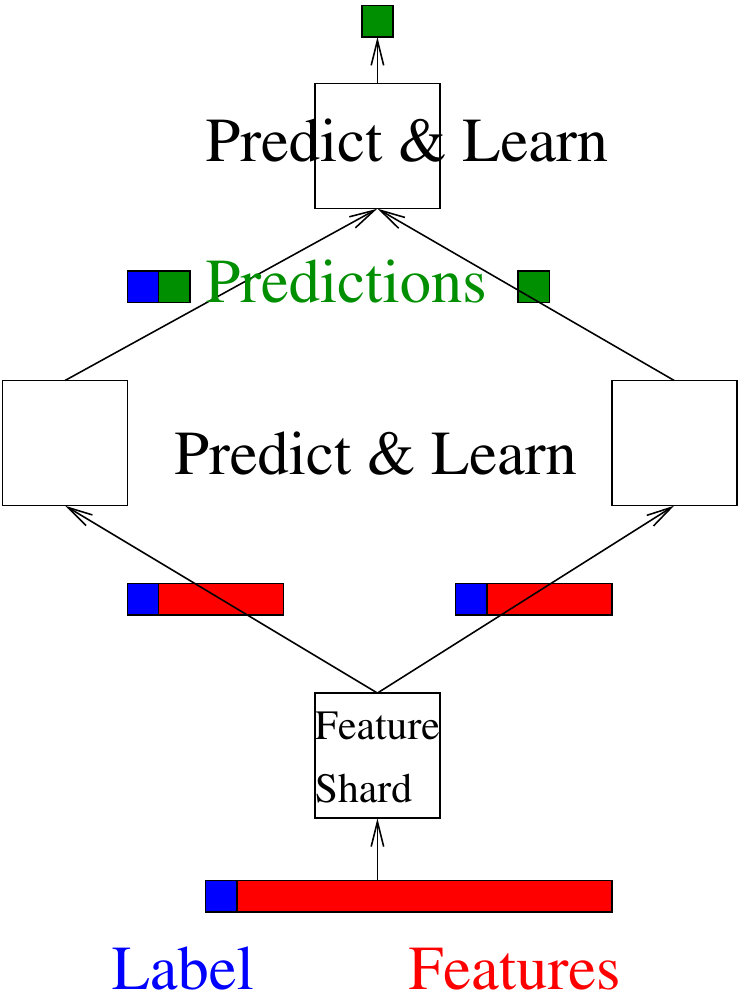}
\caption{Diagram of the parallel algorithm used in the experiments.  Step (a) starts with a full data instance.  Step (b) splits the instance's features across each shard while replicating the label to each shard.  In our experiments, the number of feature shards varies between 1 and 8. Step (c) does prediction and learning at each feature shard using only local information.  Step (d) combines these local predictions treating them as features for a final output prediction. \label{localupdate}}
\end{figure}

\begin{figure}
\begin{centering}
\begin{tabular}{cc}
\includegraphics[width=0.49\textwidth]{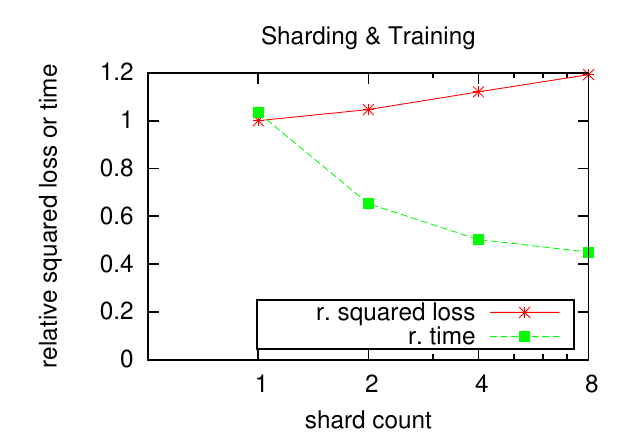} \hspace{-1cm} &
\includegraphics[width=0.49\textwidth]{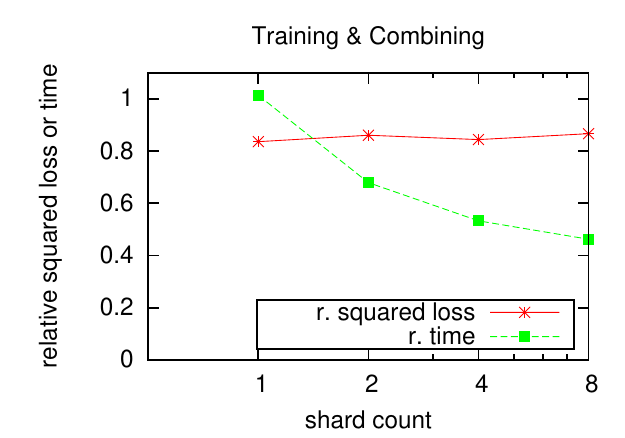} \\
(a) & (b)
\end{tabular}
\end{centering}
\caption{Plots of running time and loss versus shard count.
(a) Ratio of time and progressive squared
  loss of the shard and local train steps to a multicore
  single-machine instance of VW.  Here the squared loss is reported is
  the average of the squared losses at each feature shard,
  without any aggregation at the final output node.
(b) Ratio of time and squared loss for learning at the local
  nodes and passing information to the final output node where a final
  prediction is done. \label{fig:results}}
\end{figure}

Results are reported in Figure~\ref{fig:results}.
The first thing to note in Figure~\ref{fig:results}(a) is that there is
essentially no loss in time and precisely no loss in solution quality for
using two machines (shard count $=1$): one for a no-op shard (used just for
sending data to the other nodes) and the other for learning.
We also note that the running time does not decrease linearly in the number
of shards, which is easily explained by saturation of the network by the
no-op sharding node.
Luckily, this is not a real bottleneck, because the process of sharding
instances is stateless and (hence) completely parallelizable.
As expected, the average solution quality
across feature shards also clearly degrades with the shard count.
This is because the increase in shard count implies a decrease in the
number of features per nodes, which means each node is able to use less
information on which to base its predictions.

Upon examination of Figure~\ref{fig:results}(b), we encounter a major
surprise---the quality of the final solution substantially improves over
the single node solution since the relative squared loss is less than $1$.
We have
carefully verified this.  It is most stark when there is only one
feature shard, where we know that the solution on that shard is
identical to the single node solution.  This output prediction is then
thresholded to the interval $[0,1]$ (as the labels are either $0$ or $1$) and passed to a final prediction
node which uses the prediction as a feature and one (default) constant
feature to make a final prediction.  This very simple final prediction
step is where the large improvement in prediction quality occurs.
Essentially, because there are only two features (one is constant!),
the final output node performs a very careful calibration which
substantially improves the squared loss.

Note that one may have the false intuition that because each node does
linear prediction the final output is equivalent to a linear predictor.
This is in fact what was suggested in the previous description of the
binary tree architecture.
However, this is incorrect due to thresholding of the final prediction of
each node to the interval [0,1].

Figure~\ref{fig:results}(b) shows that the improved solution quality
degrades mildly with the number of feature shards and the running time
is again not decreasing linearly.  We believe this failure to scale
linearly is due to limitations of Ethernet where the use of many small
packets can result in substantially reduced bandwidth.  

A basic question is: How effective is this algorithm in general?
Further experiments on other datasets (below) show that the limited
representational capacity does degrade performance on many other
datasets, motivating us to consider global update rules.

\section{Global Update Rules}

So far we have outlined an architecture that lies in-between Na\"ive
Bayes and a linear model. In this section, we investigate various
trade-offs between the efficiency of local training procedure of the
previous section and the richer representation power of a linear
model trained on a single machine. Before we describe these trade-offs,
let us revisit the proof of Proposition~\ref{prop:both-fail}. In that example, the node
which gets the feature that is uncorrelated with the label learns
a zero weight because its objective is to minimize its own loss, not
the loss incurred by the final prediction at the root of the tree.
This can be easily fixed if we are willing to communicate more information
on each link. In particular, when the root of the tree has received
all the information from its children, it can send back to them some
information about its final prediction. Once a node receives some information
from its master, it can send a similar message to its children. In
what follows we show several different ways in which information can
be propagated and dealt with on each node. We call these updates global\myidx{global training}
because, in contrast to the local training of the previous section,
they use information about the \emph{final prediction of the system},
to mitigate the problems that may arise from pure local training.

\subsection{Delayed Global Update}

An extreme example of global training is to avoid local training altogether
and simply rely on the update from the master. At
time $t$ the subordinate node sends to its master a prediction $p_{t}$
using its current weights and does not use the label until time $t+\tau$
when the master replies with the final prediction $\hat{y}_{t}$ of
the system. At this point the subordinate node computes the gradient
of the loss as if it had made the final prediction itself (i.e. it
computes $\left.g_{\textrm{dg}}=\frac{\partial\ell}{\partial\langle w,x\rangle}\right|_{\langle w,x\rangle=\hat{y}_{t}}x$,
where $x$ are the node's features) and updates its weights using this
gradient.

\subsection{Corrective Update}

Another approach to global training is to allow local training when
an instance is received but use the global training rule \emph{and
undo the local training} as soon as the final prediction is received.
More formally, at time $t$ the subordinate node sends a prediction
$p_{t}$ to its master and then updates its weights using the gradient
$\left.g=\frac{\partial\ell}{\partial\langle w,x\rangle}\right|_{\langle w,x\rangle=p_{t}}x$.
At time $t+\tau$ it receives the final prediction $\hat{y_{t}}$ and
updates its weights using $\left.g_{\textrm{\textrm{cor}}}=\frac{\partial\ell}{\partial\langle w,x\rangle}\right|_{\langle w,x\rangle=\hat{y}_{t}}x-\left.\frac{\partial\ell}{\partial\langle w,x\rangle}\right|_{\langle w,x\rangle=p_{t}}x$.
The rationale for using local training is that it might be better than doing
nothing while waiting for the master, as in the case of the delayed
global update. However, once the final prediction is available, there
is little reason to retain the effect of local training and the update
makes sure it is forgotten.

\subsection{Delayed Backpropagation}

Our last update rule treats the whole tree as a composition of linear
functions and uses the chain rule of calculus to compute the gradients
in each layer of the architecture. For example, the tree of Figure~\ref{fig:hier-arch}
computes the function 
\begin{eqnarray*}
f(x) & = & w_{12}^{(2)} f_{12}(x_1,x_2) + w_{34}^{(2)} f_{34}(x_3,x_4)\\ 
f_{12}(x_1,x_2) & = & w_1^{(1)} f_1(x_1) + w_2^{(1)} f_2(x_2)\\
f_{34}(x_3,x_4) & = & w_3^{(1)} f_3(x_3) + w_4^{(1)} f_4(x_4)\\ 
f_j(x_j) & = & w_j^{(0)}x_j \quad j = 1,2,3,4.
\end{eqnarray*}
As before let $\hat y=f(x)$ and $\ell(\hat y, y)$ be our loss. Then partial 
derivatives of $\ell$ with respect to any parameter $w_i^{(j)}$ can 
be obtained by the chain rule as shown in the following examples:
\begin{eqnarray*}
\frac{\partial \ell}{\partial w_3^{(1)}} & = & 
\frac{\partial \ell}{\partial f}
\frac{\partial f}{\partial f_{34}}
\frac{\partial f_{34}}{\partial w_{3}^{(1)}} = 
\frac{\partial \ell(\hat y, y)}{\partial \hat{y}} w_{34}^{(2)} f_3 \\
\frac{\partial \ell}{\partial w_3^{(0)}} & = & 
\frac{\partial \ell}{\partial f}
\frac{\partial f}{\partial f_{34}}
\frac{\partial f_{34}}{\partial f_{3}}
\frac{\partial f_{3}}{\partial w_{3}^{(0)}} = 
\frac{\partial \ell(\hat y, y)}{\partial \hat{y}} w_{34}^{(2)} w_3^{(1)} x_3
\end{eqnarray*}
Notice here the modularity implied by the chain rule: once the node that outputs
$f_{34}$ has computed $\frac{\partial \ell}{\partial w_3^{(1)}}$ it can send to
its subordinate nodes the product $\frac{\partial \ell(\hat y, y)}{\partial
\hat{y}} w_{34}^{(2)}$ as well as the weight it uses to weigh their predictions
(i.e. $w_3^{(1)}$ in the case of the node that outputs $f_3$). The subordinate
nodes then have all the necessary information to compute partial derivatives
with respect to their own weights.  The chain rule suggests that nodes whose
predictions are important for the next level are going to be updated more
aggressively than nodes whose predictions are effectively ignored in the next
level.

The above procedure is essentially the same as the backpropagation\myidx{backpropagation} procedure,
the standard way of training with many layers of learned transformations as in
multi-layer neural networks.  In that case the composition of simple nonlinear
functions yields improved representational power. Here the gain from using a
composition of linear functions is not in representational power, as $f(x)$
remains linear in $x$, but in the improved scalability of the system.

Another difference from the backpropagation procedure is the inevitable delay
between the time of the prediction and the time of the update. In particular, at
time $t$ the subordinate node performs local training and then sends a
prediction $\bar{p}_{t}$ using the updated weights. At time $t+\tau$ it receives
from the master the gradient of the loss with respect to $\bar{p}_{t}$:
$g=\left.\frac{\partial\ell}{\partial\bar{p}_{t}}\right|_{\langle
w,x\rangle=\hat{y}_{t}}$.  It then computes the gradient of the loss with
respect to its weights using the chain rule:
$g_{\textrm{bp}}=g\cdot\frac{\partial\bar{p}_{t}}{\partial w}$.  Finally the
weights are updated using this gradient.

\subsection{Minibatch Gradient Descent}
Another class of delay-tolerant algorithms is ``minibatch''\myidx{minibatch} approaches
which aggregate predictions from several (but not all) examples before
making an aggregated update.  Minibatch has even been advocated over
gradient descent itself~\citep[see][]{ShaSinSre07}, with the basic
principle being that a less noisy update is possible after some amount
of averaging.

A minibatch algorithm could be implemented either on an example shard
organized data~\citep[as per][]{DGSX10} or on feature shard organized data. 
On an example shard based system, minibatch requires transmitting and
aggregating the gradients of all features for an example. In terms of
bandwidth requirements, this is potentially much more expensive than 
a minibatch approach on a feature shard system regardless of whether 
the features are sparse or dense. On the latter only a few bytes/example are
required to transmit individual and joint predictions at each node.
Specifically, the minibatch algorithms use global training without any 
delay: once the master has sent all the gradients in the minibatch 
to his subordinate nodes they perform an update and the next minibatch
is processed.

Processing the examples in minibatches reduces the
variance of the used gradient by a factor of $b$ (the minibatch size)
compared to computing the gradient based on one example.  However, the
model is updated only once every $b$ examples, slowing
convergence.  

Online gradient descent has two properties that might make it
insensitive to the advantage provided by the minibatch gradient:
\begin{itemize}
\item Gradient descent is a somewhat crude method: it immediately forgets
the gradient after it uses it. Contrast this with, say, bundle methods~\citep{TeoVisSmoLe09} which use
the gradients to construct a global approximation of the loss.
\item Gradient descent is very robust. In other words, gradient
descent converges even if provided with gradient estimates of 
bounded variance.
\end{itemize}
Our experiments, in the next section, confirm our suspicions and show 
that, for simple gradient descent, the optimal minibatch size is $b=1$.

\subsection{Minibatch Conjugate Gradient}

The drawbacks of simple gradient descent suggest that a gradient
computed on a minibatch might be more beneficial to a more refined
learning algorithm.  An algorithm that is slightly more sophisticated than
gradient descent is the nonlinear conjugate gradient (CG) method.
Nonlinear CG \myidx{Conjugate gradient} can be thought as gradient descent with momentum where
principled ways for setting the momentum and the step sizes are used.
Empirically, CG can converge much faster than gradient descent when
noise does not drive it too far astray.  

Apart from the weight vector $w_t$, nonlinear CG maintains a direction vector $d_t$ 
and updates are performed in the following way:
\begin{eqnarray*}
d_{t} & = & -g_t+\beta_{t} d_{t-1}\\
w_{t+1} & = & w_t + \alpha_t d_{t} 
\end{eqnarray*}
where $g_t=\sum_{\tau\in \textrm{m}(t)} \nabla_w\ell(\inner{w}{x_\tau},y_\tau)\big|_{w=w_t}$ is the gradient computed
on the $t$-th minibatch of examples, denoted by $\textrm{m}(t)$. We set $\beta_t$ according to a widely used formula~\citep{gilbert1992global}:
\[
\beta_{t} = \max\left\{0,\frac{\inner{g_t}{g_t-g_{t-1}}}{||g_{t-1}||^2}\right\}
\]
which most of the time is maximized by the second term known as the
Polak-Ribi\`ere update.  Occasionally $\beta_t = 0$ effectively
reverts back to gradient descent.  Finally, $\alpha_t$ is set by
minimizing a quadratic approximation of the loss, given by its Taylor
expansion at the current point:
\[
\alpha_t = -\frac{\inner{g_t}{d_t}}{\inner{d_t}{H_t d_t}}
\]
where $H_t$ is the Hessian of the loss at $w_t$ on the $t$-th minibatch.
This procedure avoids an expensive line search and takes advantage of 
the simple form of the Hessian of a decomposable loss which allows
fast computation of the denominator. In general $H_t=\sum_{\tau \in
\textrm{m}(t)} \ell''_\tau x_\tau x_\tau^\top$ where $\ell''_\tau =
\frac{\partial^2\ell(\hat{y},y_\tau)}{\partial \hat{y}^2}
\big|_{\hat{y}=\inner{w_t}{x_\tau}}$ is the second derivative of the loss
with respect to the prediction for the $\tau$-th example in the minibatch
$\textrm{m}(t)$.
Hence the denominator is simply 
$\inner{d_t}{H_t d_t} = \sum_{\tau \in \textrm{m}(t)} \ell_\tau'' \inner{d_t}{x_\tau}^2$.

At first glance it seems that updating $w_t$ will be an operation involving two 
dense vectors. However we have worked out a way to perform these operations
in a lazy fashion so that all updates are sparse. To see how this could work 
assume for now that $\beta_t=\beta$ is fixed throughout the algorithm and
that the $i$-th element of the gradient is non-zero at times $t_0$ and $t_1>t_0$, and 
zero for all times $\tau$ in between. We immediately see that 
\[
d_{i,\tau} = \prod_{s=t_0}^\tau \beta_s d_{i,t_0} =  d_{i,t_0} \beta^{\tau-t_0}.
\]
Hence, we can compute the direction at any time by storing a timestamp for each weight
recording its last modification time. To handle the case of varying $\beta$,
we first conceptually split the algorithm's run in phases. A new phase starts
whenever $\beta_t = 0$, which effectively restarts the CG method. Hence, within each 
phase $\beta_t \neq 0$. To compute the direction,
we need to keep track of $B_t$ the cumulative product of the $\beta$'s from 
the beginning of the phase up to 
time $t$ and use $\prod_{s=t_0}^\tau \beta_s = \frac{B_\tau}{B_{t_0}}$. 
Next, because each direction $d_t$ changes $w$ by a different amount 
$\alpha_t$ in each iteration, we must keep track of $A_t=\sum_{s=1}^t \alpha_s B_s $.
Finally, at time $t$ the update for a weight whose feature $i$ was last seen at time $\tau$ is:
\[
w_{t,i} = w_{\tau,i} + \frac{A_t-A_{\tau-1}}{B_\tau} d_{\tau,i}.
\]

\subsection{Determinizing the Updates}

In all of the above updates, delay plays an important role. Because
of the physical constraints of the communication, the delay $\tau$
can be different for each instance and for each node. This can have
an adverse effect on the reproducibility\myidx{reproducibility} of our results. To see this
it suffices to think about the first time a leaf node receives a response.
If that varies, then the number of instances for which this node will
send a prediction of zero to its master varies too. Hence the weights
that will be learned are going to be different. To alleviate this
problem and ensure reproducible results our implementation takes special
care to impose a deterministic schedule of updates. This has also
helped in the development and debugging of our implementation. Currently,
the subordinate node switches between local training on new instances
and global training on old instances in a round robin fashion, after
an initial period of local training only, that maintains $\tau=1024$
(which is half the size of the node's buffer). In other words, the
subordinate node will wait for a response from its master if doing
otherwise would cause $\tau >1024$. It would
also wait for instances to become available if doing otherwise would
cause $\tau<1024$, unless the node is processing the last 1024 instances
in the training set.

\section{Experiments}

Here we experimentally compare the predictive performance of the local, 
the global, and the centralized update rules. We derived
classification tasks from the two data sets described in
Table~\ref{table:datasets}, trained predictors using each training
algorithm, and then measured performance on separate test sets. 
For each algorithm, we perform a separate search for the best
learning rate schedule of the form $\eta_t=\frac{\lambda}{\sqrt{t+t_0}}$
with $\lambda \in \{2^i\}_{i=0}^9$, $t_0 \in \{10^i\}_{i=0}^6$. We report
results with the best learning rate we found for each algorithm and task. For 
the minibatch case we report a minibatch size of $1024$ but we also 
tried smaller sizes even though there is little evidence that
they can be parallelized efficiently. Finally we report the 
performance of a centralized stochastic gradient descent (SGD)
which corresponds to minibatch gradient descent with a batch size of 1.

\begin{table}
\begin{tabular}{|l|c|c|}
Name & \# training data & \# testing data \\
\hline
RCV1 & $780$K & $23$K \\
Webspam & $300$K & $50$K \\
\hline
\end{tabular}
\caption{Description of data sets in global experiments.}
\label{table:datasets}
\end{table}

We omit results for the Delayed Global and Corrective update rules because
they have serious issues with delayed feedback. Imagine trying to 
control a system (say, driving a car) that responds to actions after much
delay. Every time an action is taken (such as steering in one direction) 
it is not clear how much it has affected the response of the system. If 
our strategy is to continue performing the same action until its effect is 
noticeable, it is likely that by the time we receive all the delayed feedback, 
we will have produced an effect much larger than the desired.
To reduce the effect we can try to undo our action which of course can 
produce an effect much smaller than what was desirable. The system then
oscillates around the desired state and never converges there. This is 
exactly what happens with the delayed global and corrective update rules.
Delayed backpropagation is less susceptible to this problem because 
the update is based on both the global and the local gradient. Minibatch
approaches completely sidestep this problem because the information
they use is always a gradient at the current weight vector. 

In Figure~\ref{fig:global_results} we report our results on each data
set.  We plot the test accuracy of each algorithm under different
settings.  ``Backprop x8'' is the same as backprop where the gradient
from the master is multiplied by 8 (we also tried 2, 4, and 16 and
obtained qualitatively similar results)---we tried this variant as a
heuristic way to balance the relative importance of the backprop
update and that of the local update.  In the first row of
Figure~\ref{fig:global_results}, we show that the performance of both
local and global learning rules degrades as the degree of
parallelization (number of workers) increases.  However, this effect
is somewhat lessened with multiple passes through the training data
and is milder for the delayed backprop variants, as shown in in the
second row for the case of 16 passes.  In the third and fourth rows,
we show how performance improves with the number of passes through the
training data, using 1 worker and 16 workers. Notice that SGD,
Minibatch, and CG are not affected by the number of workers as they
are global-only methods. Among these methods SGD dominates CG which in
turn dominates minibatch. However, SGD is not parallelizable while
minibatch CG is.

\begin{figure}
\begin{centering}
\begin{tabular}{c|c}
\includegraphics[width=0.49\textwidth]{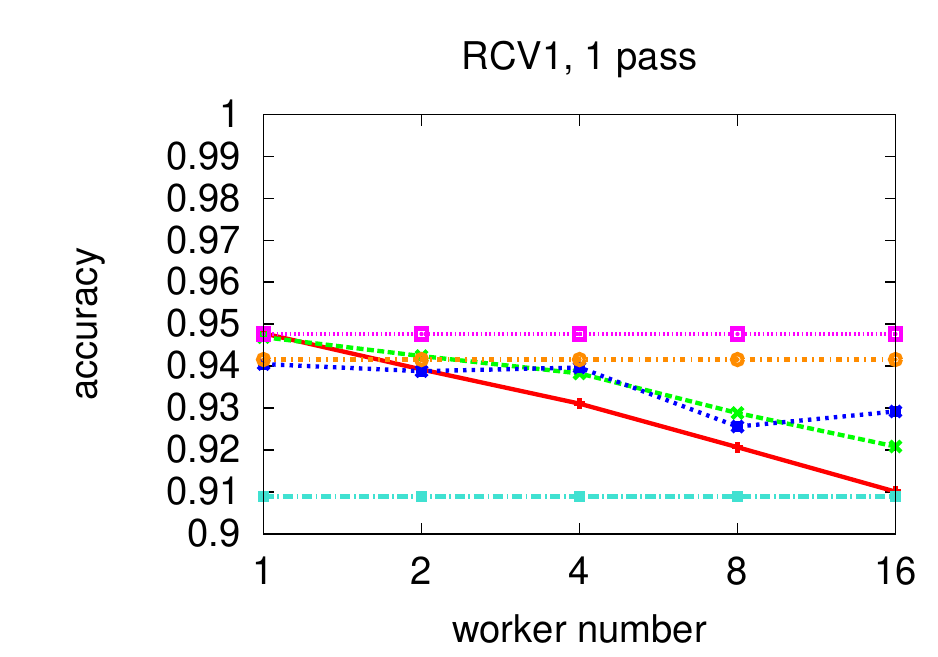} \hspace{-1cm} &
\includegraphics[width=0.49\textwidth]{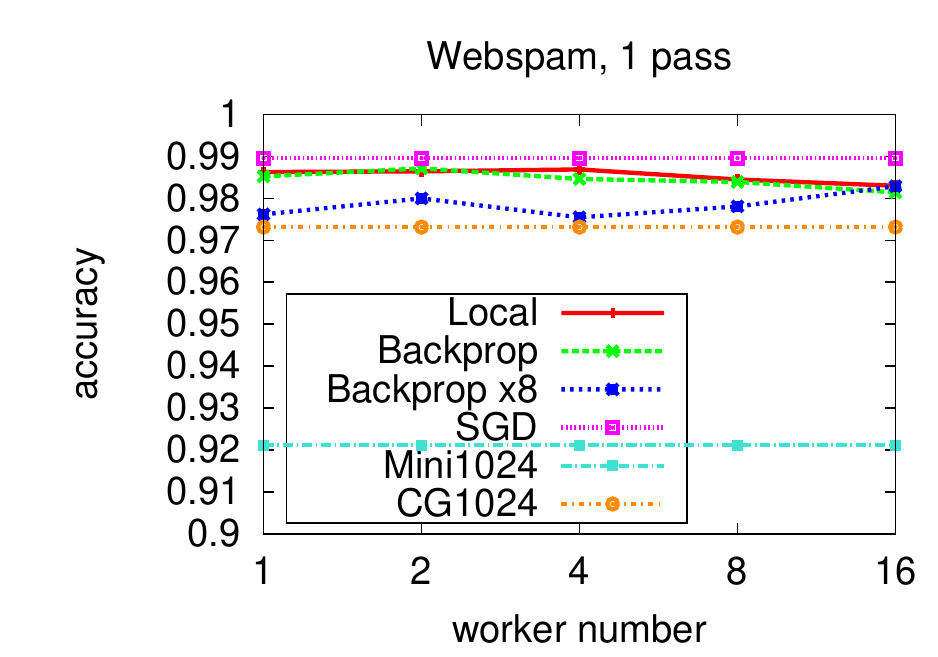}
\\
\includegraphics[width=0.49\textwidth]{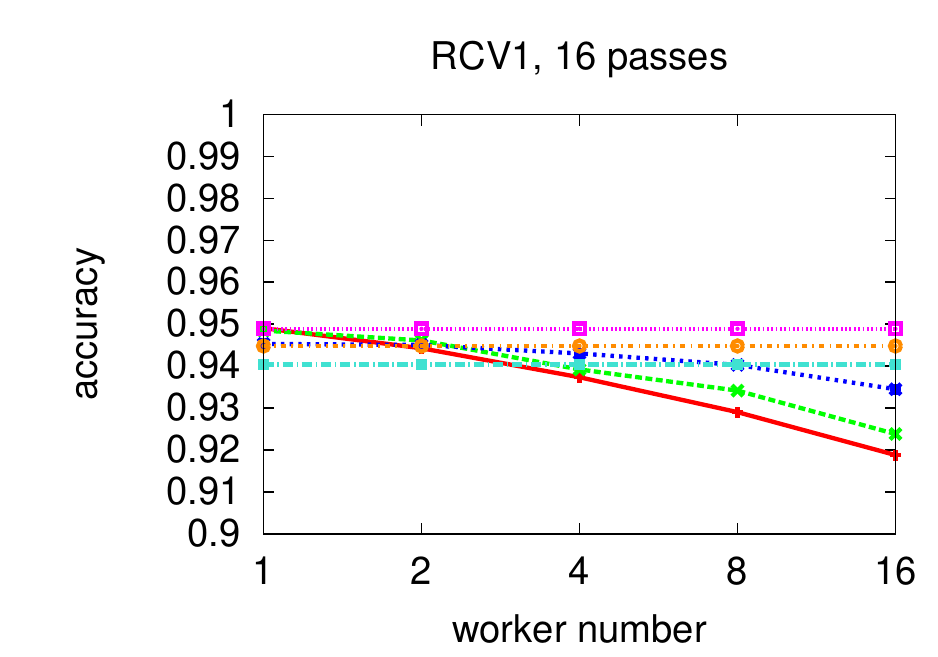} \hspace{-1cm} &
\includegraphics[width=0.49\textwidth]{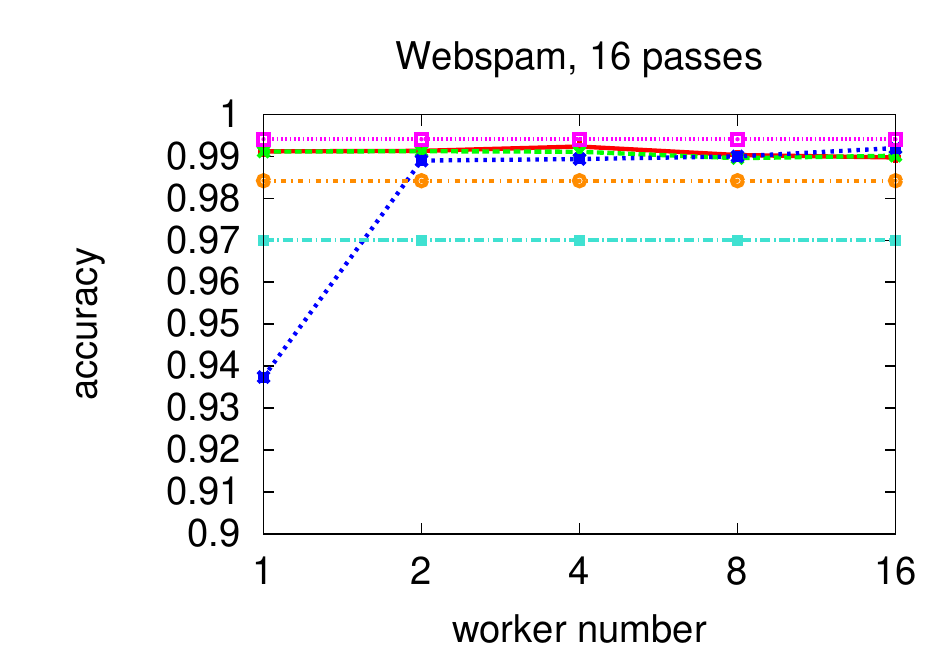}
\\
\includegraphics[width=0.49\textwidth]{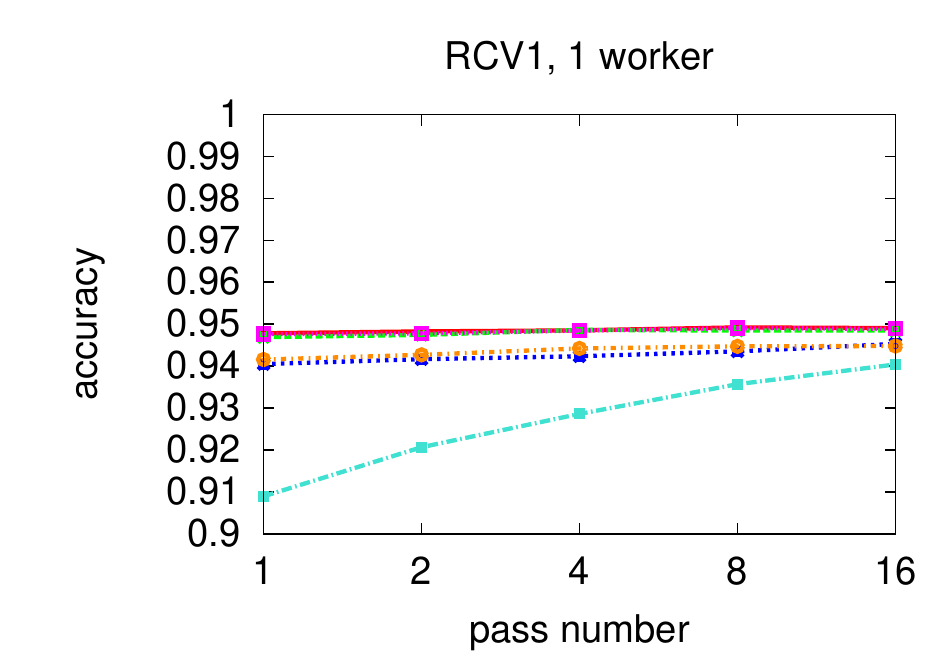} \hspace{-1cm} &
\includegraphics[width=0.49\textwidth]{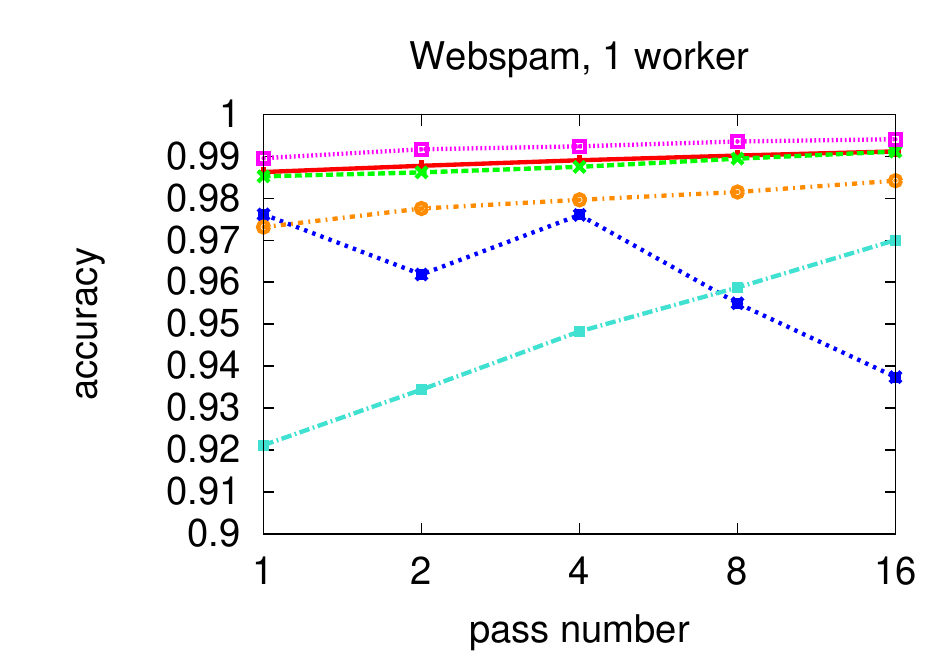}
\\
\includegraphics[width=0.49\textwidth]{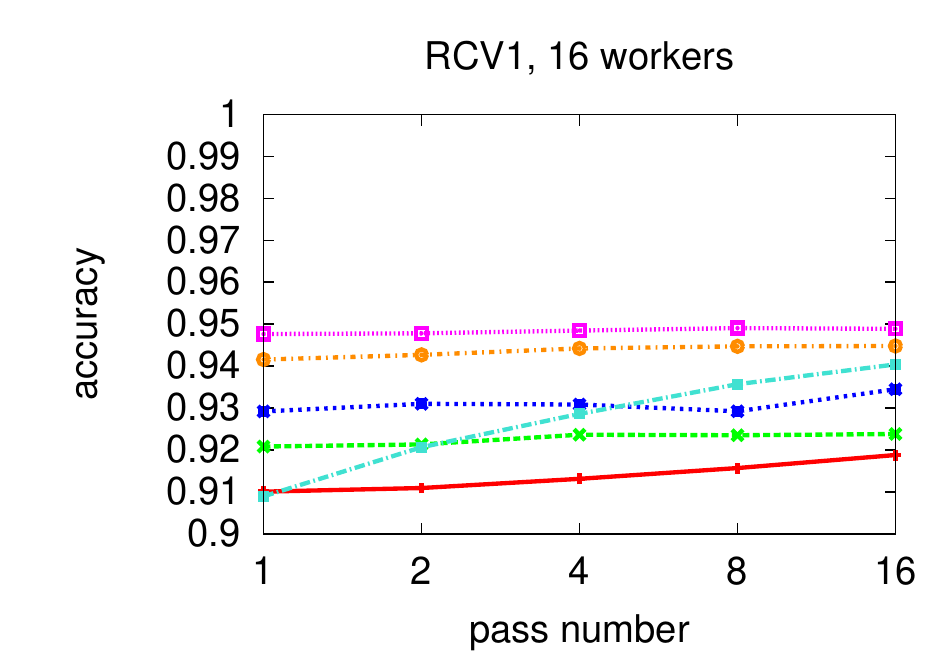} \hspace{-1cm} &
\includegraphics[width=0.49\textwidth]{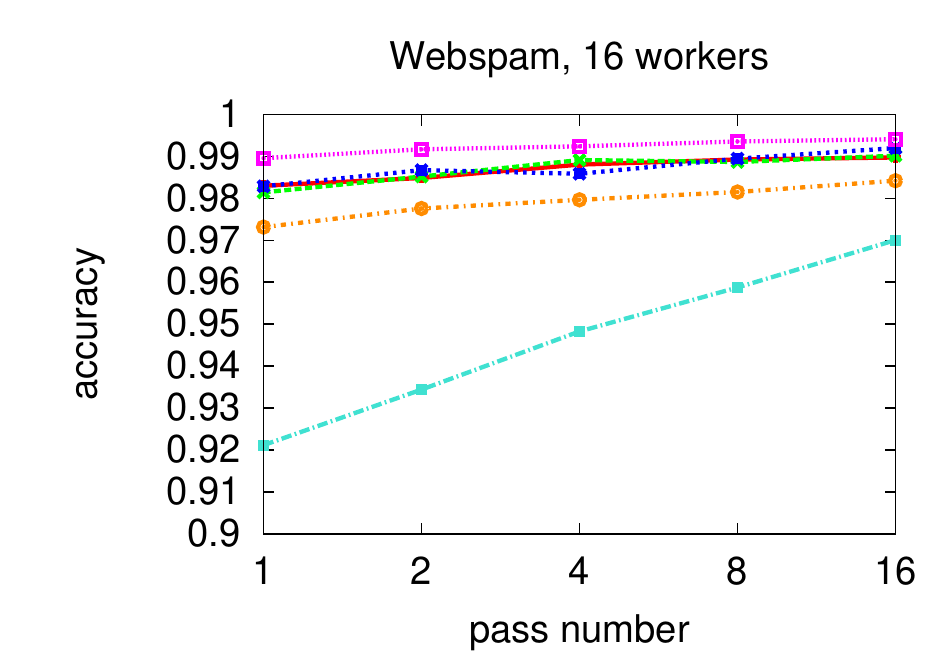}
\end{tabular}
\end{centering}
\caption{
Experimental results compare global to local learning rules.
In the first two rows, we see how performance degrades for various rules as the number of workers increases.
In the last two rows, we show how performance changes with multiple passes.
All plots share the same legend, shown in the top right plot.
}
\label{fig:global_results}
\end{figure}

\section{Conclusion}

Our core approach to scaling up and parallelizing learning is to first
take a very fast learning algorithm, and then speed it up even more.
We found that a core difficulty with this is dealing with the problem of
delay in online learning.
In adversarial situations, delay can reduce convergence speed by the delay
factor, with no improvement over the original serial learning algorithm.

We addressed these issues with parallel algorithms based on
feature sharding.
The first is simply a very fast multicore algorithm which manages to avoid
any delay in weight updates by virtue of the low latency between cores.
The second approach, designed for multinode settings, addresses the latency
issue by trading some loss of representational power for local-only
updates, with the big surprise that this second algorithm actually
\emph{improved} performance in some cases.
The loss of representational power can be addressed by incorporating
global updates either based on backpropagation on top of the local
updates or using a minibatch conjugate gradient method; experimentally, 
we observed that the combination of local and global updates can improve 
performance significantly over the local-only updates.

The speedups we have found so far are relatively mild due to working
with a relatively small number of cores, and a relatively small number
of nodes.  Given that we are starting with an extraordinarily fast
baseline algorithm, these results are unsurprising.  A possibility
does exist that great speedups can be achieved on a large cluster of
machines, but this requires further investigation.


\bibliography{bibfile}
\bibliographystyle{cambridgeauthordate}

\end{document}